\documentclass{article} % Anonymized submission
\title{Efficient improper learning for online logistic regression}
\usepackage{arxiv}
\usepackage[utf8]{inputenc} % allow utf-8 input
\usepackage[T1]{fontenc}    % use 8-bit T1 fonts
\usepackage{amsthm,amsmath,amssymb}
\usepackage{booktabs}
\usepackage{graphicx}
\usepackage{algpseudocode,algorithm2e}
\usepackage{enumitem}

\newcommand{\namealgo}{AIOLI }

\author{
  Rémi Jézéquel 
   \And
  Pierre Gaillard
  \And
  Alessandro Rudi \\
}

\date{INRIA - Département d’Informatique de l’École Normale Supérieure \\
PSL Research University \\
Paris, France}

\newcommand{\cX}{\mathcal{X}}
\newcommand{\cB}{\mathcal{B}}
\newcommand{\E}{\mathbb{E}}
\newcommand{\R}{\mathbb R}
\newcommand{\htheta}{\hat \theta}
\newcommand{\ie}{\textit{i.e.,}}
\newcommand{\argmin}[1]{\mathop{\operatorname{argmin}}_{#1}}
\newcommand{\deff}{d_{\textrm{eff}}}
\newcommand{\cY}{\mathcal{Y}}
\renewcommand{\hat}{\widehat}
\renewcommand{\epsilon}{\varepsilon}
\renewcommand{\le}{\leq}
\renewcommand{\ge}{\geq}

\renewcommand{\geq}{\geqslant}
\renewcommand{\leq}{\leqslant}
\newcommand{\ch}{\operatorname{cosh}}

\newtheorem{theorem}{Theorem}
\newtheorem{lemma}[theorem]{Lemma}
\newtheorem{corollary}[theorem]{Corollary}
\renewenvironment{proof}
{\par\noindent{\bfseries\upshape Proof\ }}
{\hfill\ensuremath{\square}}

\newenvironment{sketchproof}{\paragraph{Sketch of proof}}{\hfill$\blacksquare$}

\begin{document}

\maketitle

\begin{abstract}
We consider the setting of online  logistic regression and consider the regret with respect to the $\ell_2$-ball of radius $B$. It is known (see~\cite{hazan2014logistic}) that any proper algorithm which has logarithmic regret in the number of samples (denoted $n$) necessarily suffers an exponential multiplicative constant in $B$. In this work, we design an efficient improper algorithm that avoids this exponential constant while preserving a logarithmic regret. 

Indeed, \cite{foster2018logistic} showed that the lower bound  does not apply to improper algorithms and proposed a strategy based on exponential weights with prohibitive computational complexity. Our new algorithm based on regularized empirical risk minimization with surrogate losses satisfies a regret scaling as $O(B\log(Bn))$ with a per-round time-complexity of order $O(d^2 + \log(n))$.
\end{abstract}

\section{Introduction}

In online learning, a learner sequentially interacts with an environment and tries to learn based on data observed on the fly \cite{cesa2006prediction,hazan2016introduction}. More formally, at each iteration $t \ge 1$, the learner receives some input $x_t$ in some input space $\mathcal{X}$; makes a prediction $\hat y_t$ in a decision domain $\hat \cY$ and the environment reveals the output $y_t \in \mathcal{Y}$. The inputs $x_t$ and the outputs $y_t$ are sequentially chosen by the environment and can be arbitrary. No stochastic assumption (except boundedness) on the data sequence $(x_t,y_t)_{1\leq t\leq n}$ is made.    The accuracy of a prediction  $\hat y_t \in \hat \cY$ at instant $t\geq 1$ for the outcome $y_t \in \cY$ is measured through a loss function $\ell:\hat \cY \times \cY \to \R$. The learner aims at minimizing his cumulative regret
\begin{equation}
    R_n(f) = \sum\limits_{t=1}^n \ell \big(\hat y_t, y_t\big) - \sum\limits_{t=1}^n \ell\big(f(x_t),y_t\big) \,,
    \label{eq:regret_def}
\end{equation}
uniformly over all functions $f$ in a reference class of functions $\mathcal{F}$. All along this paper, we will consider the more specific setting of online logistic regression for binary classification. The latter corresponds to binary outputs $y_t \in \mathcal{Y} = \{-1,1\}$, real decisions $\hat y_t \in \hat \cY = \mathbb{R}$, the logistic loss function $\ell: (\hat y_t, y_t) \mapsto \log (1 + e^{-y_t \hat y_t})$ and the reference class $\mathcal{F} = \{x \mapsto \theta^\top x; \theta \in \mathcal{B}(\mathbb{R}^d,B) \}$ of linear functions in the $\ell_2$-ball of radius $B>0$.

\medskip
Logistic regression, which dates back to \cite{berkson1944application}, has been widely studied in the past decades both in the statistical and online setting. It allows to estimate conditional probabilities and is heavily used in practice for multi-class and binary classification. Since the statistical literature is abundant, we highlight here only the key existing approaches for online logistic regression that are relevant for the present work. Using basic properties of the logistic loss, classical algorithms from Online Convex Optimization can be used to minimize the regret~\eqref{eq:regret_def}. On the one hand,
remarking that the logistic loss is convex and Lipschitz, one may use Online Gradient Descent (OGD) of \cite{zinkevich2003online}, which guarantees a regret of order $O(B\sqrt{n})$. On the other hand, using that the logistic loss is $e^{-B}$-exp concave, one can use Online Newton Step (ONS) from \cite{hazan2007logarithmic} which achieves a logarithmic regret of order $O(de^B \log(n))$.

In view of this results, one could wonder if obtaining a better dependence on the number of samples comes with an exponential deterioration on the multiplicative constant in $B$. \cite{hazan2014logistic} considered this exact question and showed that indeed any proper algorithm in the regime $n = O(e^B)$ has at least a worst-case regret of order $\Omega(B^{2/3} n^{1/3})$ for one dimensional inputs.  Therefore any bound of the form $O(B\log(n))$ is impossible for proper algorithms. We recall that an algorithm is called proper if its prediction function $\hat f_t:\cX \to \hat \cY$ is in the reference class $\mathcal{F}$. In other words, it means that for all $t\geq1$, the prediction is of the form $\hat y_t = \hat f_t(x_t)$ with $f_t \in \mathcal{F}$ independent of $x_t$ (i.e., the prediction function is linear in $x_t$ in our case).

However, it was recently shown that this lower-bound does not apply to improper algorithms \cite{foster2018logistic}. Indeed, based on the simple observation that the logistic loss is 1-mixable (see \cite{vovk1998game} for the definition), they could apply Vovk’s Aggregating Algorithm \cite{vovk1998game} which leverages mixability to achieve a regret of order $O(d\log(Bn))$. In particular, they showed that for online logistic regression improper algorithms can significantly outperform proper algorithms by proving a doubly-exponential improvement on the constant $B$. Yet, the complexity of their algorithm, while being polynomial in $d$ and $n$ is highly prohibitive making the algorithm infeasible in practice. Vovk’s Aggregating Algorithm is indeed based on a continuous version of the exponentially weighted average forecaster. To output a prediction one needs to approximate an integral over the  $d$-dimensional ball which requires the use of MCMC approximations. Using the projected Langevin Monte Carlo sampler from \cite{bubeck2018sampling}, they record a computation time of $O(B^{6} n^{12}(Bn+d)^{12})$.

\medskip
This is the starting point of this work. Can we achieve similar performance in online logistic regression with  practical computational complexity? Recently, some works attacked this question for logistic regression in the batch statistical setting with i.i.d. data only. \cite{marteau2019beyond} considered the classical regularized empirical risk minimizer (ERM). Though the latter is proper, using generalized self-concordance properties they could avoid the exponential constant in $B$ under additional assumptions including a well-specified problem, capacity and source conditions. In parallel and independently of this work, \cite{mourtada2019improper} have also designed a practical improper algorithm in the statistical setting based on ERM with an improper regularization using virtual data. They could provide an upper-bound on the excess risk in expectation of order $O((d + B^2 )/n)$. However, they left open the question of achieving it in an online setting.

\paragraph{Contributions} In this paper, we introduce a new practical improper algorithm, that we call \namealgo (Algorithmic efficient Improper Online LogIstic regression), for online logistic regression. The latter is based on Follow The Regularized Leader (FTRL) \cite{mcmahan2011follow} with surrogate losses. \namealgo takes inspiration from the Azoury-Warmuth-Vovk forecaster (also named non-linear Ridge regression or AWV) from \cite{vovk2001competitive} and \cite{azoury2001relative} which adds a non-proper penalty based on the next input $x_t$ and from Online Newton Step \cite{hazan2007logarithmic} which leverages the exp-concavity of logistic regression to achieve logarithmic regret. The per-round space and time complexity of \namealgo is of order $(O(nd^2 + n\log(n))$ which is close to the one of ONS and greatly improves the ones of  \cite{foster2018logistic}. 

We provide in Theorem~\ref{thm:main_theorem} an upper-bound on the regret of \namealgo of the order $O(dB\log(Bn))$. This makes \namealgo provably better than any proper algorithm in the regime where $n = O(e^B)$. 
To illustrate our results, we provide simulations on synthetic data generated by the \textit{adversarial} distribution of \cite{hazan2014logistic} that show that, contrary to classical FTRL, the regret of \namealgo is indeed logarithmic. We summarize in Table~\ref{tab:rates} the rates and per-round computational complexities of the key-algorithms for logistic regression. 

\medskip
In addition to introducing \namealgo, we make  two technical contributions that we believe to  be of their own interests.
Our first technical contribution is based on the simple observation that 
    the logistic function $x \mapsto \log(1+e^{-x})$ is \textit{only} $e^{-B}$-exp concave on $[-B,B]$ when $x$ is close to $-B$. For the rest of the range (typically $x \in [0,B]$), far better exp-concavity parameters (that we also refer to as \textit{curvature}) may be achieved. Therefore, contrary to ONS which uses the worst-case value for the curvature, we consider quadratic approximations of the logistic loss with data-dependent curvature parameters. These approximations are used as surrogate losses minimized by~\namealgo.

 Our second technical contribution is to use an improper regularization that allows us to not pay the worst \textit{curvature} but only the one for $x$ close to $0$. This regularization is inspired from the non-linear Ridge forecaster of~\cite{azoury2001relative} and \cite{vovk2001competitive}. Typically, when a new input $x_t$ is observed by the learner, the latter can use it to regularize more in the direction of $x_t$. If the learner knew the next output $y_t$ a good regularization would be to add the loss $\ell(f(x_t),y_t)$ when computing FTRL. Yet $y_t$ is unknown and the learner must use a regularization independent of $y_t$. The  non-linear Ridge forecaster consists in replacing $y_t$ by $0$. Instead, \namealgo regularizes by adding both $\ell(f(x_t),1)$ and $\ell(f(x_t),-1)$ to the empirical loss to be minimized. The important phenomena is that the dominant regularization is $\ell(f(x_t),y_t)$ if $y_t \htheta_t^\top x_t \ll 0$, that is when the algorithm makes a large error. It is worth emphasizing that this regularization depends on the next input $x_t$ and thus makes our algorithm improper. We believe this type of regularization to be new for online logistic regression and have significant interest to inspire future work.

\begin{table}
    \centering
    \begin{tabular}{ccccc}
        \toprule
        Algorithm & OGD & ONS & \cite{foster2018logistic} &  \namealgo \\
        \midrule
        Regret  & $B\sqrt{n}$ & $de^B\log(n)$ & $d\log(Bn)$ & $dB\log(Bn)$\\
        Total complexity & $nd$ & $nd^3$ & $B^{6} n^{12}(Bn+d)^{12}$ & $nd^2 + n\log(n)$ \\ 
    \bottomrule
    \end{tabular}
    \caption{Regret bounds and computational complexities (in $O(\cdot)$) of relevant algorithms}
    \label{tab:rates}
\end{table}

\paragraph{Setting and notation} We recall the setting and introduce the main notations that will be used all along the paper. Our framework is formalized as a sequential game between a learner and an environment. At each forecasting instance $t \ge 1$, the learner is given an input $x_t \in \cX \subseteq \mathcal{B}(\R^d,R)$ for some radius $R>0$ and dimension $d\geq 1$;  chooses a vector $\hat \theta_t \in \R^d$ (possibly based on the current input $x_t$ and on the past information $x_1,y_1,\dots,x_{t-1},y_{t-1}$); and makes the prediction $\hat y_t = \hat \theta_t^\top x_t \in \R$. Then, the environment chooses $y_t \in \{-1,1\}$; reveals it to the learner which  incurs the loss $\ell_t(\htheta_t) = \ell(\htheta_t^\top x_t, y_t)$ where for all $\theta \in \R^d$,
\[ 
    \ell\big(\theta^\top x_t, y_t\big) = \log\big(1 + e^{-y_t \theta^\top x_t}\big) \,. 
\]
%We also define for any $y \in \{-1,1\}$ the  loss $\ell_t^y(\theta) = \ell(\theta^\top x_t, y)$ that the parameter $\theta \in \R^d$ would have incurred for the output $y$. In particular, we have $\ell_t(\theta) = \ell_t^{y_t}(\theta)$. Throughout the paper, we will use the convention for every notations that if the index $y$ is omitted then it means $y = y_t$. 
Moreover, the gradients of the loss functions at the estimator will be denoted as $g_t = \nabla \ell_t(\htheta_t) \in \R^d$. We recall that the goal of the learner is to minimize the cumulative regret 
\[ 
    R_n(\theta) = \sum\limits_{t=1}^n \ell\big(\hat \theta_t^\top x_t, y_t\big) - \sum\limits_{t=1}^n \ell\big(\theta^\top x_t,y_t\big) \,,
\]
uniformly over all $\theta \in \mathcal{B}(\R^d,B)$ and all possible sequences $(x_1,y_1),...,(x_n,y_n) \in \cX \times \cY$.

\section{Main contributions}

This section gathers the main contributions of the present paper. Essentially, we introduce in Section~\ref{sec:algorithm} our new algorithm for online logistic regression. In Section~\ref{sec:theorems}, we prove the corresponding upper-bounds on the regret and we provide an efficient implementation in Section~\ref{sec:implementation}.

\subsection{\namealgo: a new algorithm for online logistic regression} 
\label{sec:algorithm}

We introduce here and briefly describe a new algorithm  \namealgo for online logistic regression. More details on the underlying ideas are provided in Section~\ref{sec:technical_contrib}. \namealgo is based on FTRL which is applied on surrogate quadratic losses and with an additional improper regularization. It requires the knowledge of three hyper-parameters: a regularization parameter $\lambda >0$, the diameter of the input space $R>0$ and the diameter of the reference class $B>0$. At each forecasting instance $t\geq 1$, we first define the following quadratic approximations of the past losses for $1\leq s < t$ that are defined by:  for all $\theta \in \R^d$
\begin{equation}
    \label{eq:surrogate_loss_def}
    \hat \ell_s(\theta) = \ell_s(\htheta_s) + g_s^{\top} (\theta - \htheta_s) + \frac{\eta_s}{2} (\theta - \htheta_s)^\top g_s g_s^{\top} (\theta - \htheta_s) \,, \qquad \text{with} \quad \eta_s =  \frac{e^{y_s \hat y_s}}{1+BR} \,.
\end{equation}
This approximation is discussed more in details in Section~\ref{sec:quadratic_approx}. The main point to be noticed is that the curvature parameters $\eta_s$ are adapted to the predictions of the algorithms $\hat y_s$ in contrast to ONS which uses the worst-case values $e^{-B}$ for all $s\geq 1$. 
% {\color{red} Then, \namealgo computes the FTRL applied on the surrogates losses
% \begin{equation}\label{eq:def_theta_base}
%     \theta_t = \argmin{\theta \in \R^d} \left\{ \sum\limits_{s=1}^{t-1} \hat \ell_s(\theta) + \lambda \|\theta\|^2 \right\}\,.
% \end{equation}
% The parameter $\theta_t$ is then used to make a guess $z_t =  -\sign(\theta_t^\top x_t)$ of what would be the next worst-case output. In other words, the output that would provide a wrong prediction. Finally, it computes the parameter
% \begin{equation}
% \label{eq:def_estimator_old}
% \htheta_t = \argmin{\theta \in \R^d} \left\{ \sum\limits_{s=1}^{t-1} \hat \ell_s(\theta) + \ell(\theta^\top x_t,z_t) + \lambda \|\theta\|^2 \right\}
% \end{equation}
% and predicts $\hat y_t = \hat \theta_t^\top x_t$.}

Then, \namealgo computes the following estimator
\begin{equation}
\label{eq:def_estimator}
\htheta_t = \argmin{\theta \in \R^d} \left\{ \sum\limits_{s=1}^{t-1} \hat \ell_s(\theta) + \ell(\theta^\top x_t,1) + \ell(\theta^\top x_t,-1) + \lambda \|\theta\|^2 \right\}
\end{equation}
and predicts $\hat y_t = \hat \theta_t^\top x_t$. 

We point out that both regularization terms use the original logistic loss $\ell$ and not its approximation $\hat \ell_t$. Still, $\hat \ell_t(\hat \theta_t)$ equals  $\ell(\hat \theta_t^\top x_t, y_t)$.
Remark that this algorithm is indeed improper since $\hat \theta_t$ depends on the next input $x_t$ which implies a non-linear prediction $\hat y_t = \hat \theta_t^\top x_t$ (see Figure~\ref{fig:my_label}). We propose in Section~\ref{sec:implementation} an efficient scheme to sequentially compute $\hat \theta_t$ with low computational and storage complexities.

% The surrogate loss functions $\hat \ell_s$ are quadratic approximations of $\ell_s$ in $\htheta_s$. Apart the second term ($\ell_t^{z_t}(\theta)$), this estimator looks like ridge regression on quadratic approximations of the logistic loss.  It's known that such an estimator will suffer a logarithmic regret but with a multiplicative constant which depends on the \textit{curvature} parameter of the quadratics i.e. on the inverse of the smallest eigen value of the hessian (TRUE ???)). In the case of the logistic regression it's possible that this \textit{curvature} parameter could be as bad as $e^{-BR}$. This happens when $y \theta^\top x \ll 0$. It corresponds to the left side of figure ??? when $y = 1$. However, if we add an \textit{improper} regularization $\ell_t^{z_t}(\theta)$ it will push back our estimator to the center essentially avoiding to suffer any regret if $y_t = 1$. Now, if $y_t = -1$, we see that we are in a far better situation because the logistic loss is almost flat when $y \theta^\top x \gg 0$.

\begin{figure}[th]
    \centering
    \includegraphics[width=.6\textwidth]{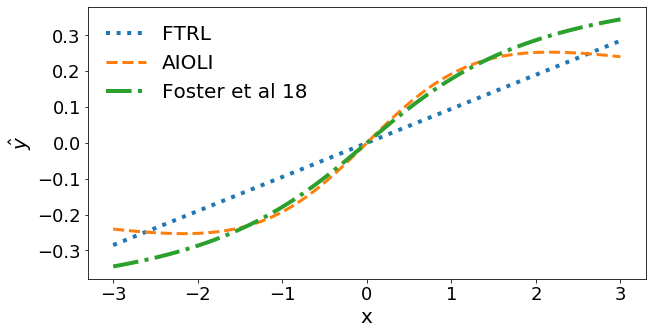}
    \caption{Example of prediction functions obtained by FTRL, \namealgo and the algorithm of \cite{foster2018logistic}.}
    \label{fig:my_label}
\end{figure}

\subsection{Logarithmic upper-bound on the regret without exponential constants}

\label{sec:theorems}

We state now our main theoretical result which is an upper bound on the regret suffered by \namealgo.

\begin{theorem}
\label{thm:main_theorem}
Let $\lambda, R, B>0$ and  $d, n\geq 1$. Let $(x_1,y_1),...,(x_n,y_n) \in \cX \times \cY$ be an arbitrary sequence of observations. \namealgo (as defined in Equation~\eqref{eq:def_estimator}) run with regularization parameter $\lambda > 0$ satisfies the following upper-bound on the regret
\[ 
R_n(\theta) \le \lambda \|\theta\|^2 + d(1 + BR) \log \left( 1 + \frac{n R^2}{8d(1+BR)\lambda} \right) \,,
\]
for all $\theta \in \cB(\R^d,B)$. In particular, by choosing $\lambda = \frac{1}{B^2}$, it yields for all $\theta \in \cB(\R^d,B)$
\begin{equation}
    \label{eq:regret_bound}
R_n(\theta) \le d(1 + BR)  \log \left(1 + \frac{nB^2R^2}{8d(1+BR)}\right) + 1  \,.
\end{equation}
\end{theorem}

This theorem is a consequence of the more general theorem \ref{thm:main_theorem_approx} which is deferred to Appendix~\ref{app:proof_main_thm}. We only highlight below the key ingredients of the proof. Theorem~\ref{thm:main_theorem} states that the regret of \namealgo is logarithmic in $n$ with a multiplicative constant of order $dB$ which is an exponential improvement in $B$ over the one achieved by proper algorithms such as ONS \cite{hazan2007logarithmic}. Yet, our regret upper-bound is weaker than the one of \cite{foster2018logistic} which is of order $O(d\log(Bn))$. Their algorithm however requires a prohibitive time complexity of order $O(B^6n^{12}(Bn+d)^{12})$ through complex MCMC procedures. We leave for future work the question weather their regret is achievable by our algorithm or not.

\begin{sketchproof}
The proof of the theorem is based on two main steps: 1) we upper-bound the cumulative regret using the true losses by the cumulative regret using the quadratic surrogate losses; 2) we can then follow (with some adjustments) the analysis for online linear regression with squared loss of \cite{azoury2001relative} and \cite{vovk2001competitive} (see also the proof of \cite{gaillard2018uniform}). Fix $\theta \in \cB(\R^d,B)$.

\medskip \noindent
\emph{Step~1.} The first step (i.e., the upper-bound of the regret with the surrogate regret) uses the key Lemma~\ref{lem:quad_approx_lower}, which implies that the quadratic surrogate loss are lower-bounds on the logistic losses. That is,
\[
     \forall t\geq 1, \qquad  \hat \ell_t(\theta) \leq \ell_t(\theta).
\]
Using that by definition (see Equation~\eqref{eq:surrogate_loss_def}) we also have $\hat \ell_t(\hat \theta_t) = \ell_t(\hat \theta_t)$ for all $t \geq 1$, this entails $\ell_t(\hat \theta_t) - \ell_t(\theta) \leq \hat \ell_t(\hat \theta_t) - \hat \ell_t(\theta)$, which implies
\[ 
    R_n(\theta) = \sum\limits_{t=1}^n \ell_t(\htheta_t) - \sum\limits_{t=1}^n \ell_t(\theta) \le \sum\limits_{t=1}^n \hat \ell_t(\htheta_t) - \sum\limits_{t=1}^n \hat \ell_t(\theta) = \hat R_n(\theta)  \,.
\]

\medskip \noindent
\emph{Step 2.} Using that the surrogates losses $\hat \ell_t$ are quadratic, the second part of the proof follows the one of \cite{gaillard2018uniform} for online least square regression. After technical linear algebra computation, this leads to
\[
    \hat R_n(\theta) \leq \sum_{t=1}^n (\theta_{t+1} - \htheta_t)^\top A_t (\theta_{t+1} - \htheta_t) - (\theta_t - \htheta_t)^\top A_{t-1} (\theta_t - \htheta_t) \,,
\]
where $A_t = \lambda I + \sum_{s=1}^t \frac{\eta_s}{2} g_s g_s^\top$ and we recall that $g_s = \nabla \hat \ell_s(\hat \theta_s)$ and $\eta_t = e^{y_t  \hat y_t}/ (1 + BR)$. 
% {\color{red} Now, we consider two cases depending on whether or not the guess $z_t$ of the algorithm for the next output $y_t$ was correct. If the guess was correct then $z_t = y_t$, this means $\hat \theta_t = \theta_{t+1}$ and 
% \[
%     (\theta_{t+1} - \htheta_t)^\top A_t (\theta_{t+1} - \htheta_t) - (\theta_t - \htheta_t)^\top A_{t-1} (\theta_t - \htheta_t) = - (\theta_t - \htheta_t)^\top A_{t-1} (\theta_t - \htheta_t) \leq 0\,.
% \]
% Otherwise, the guess was incorrect which implies $y_t = - z_t = \sign(\theta_t^\top x_t)$. But, using Lemma~\ref{lem:bound_eta}, we show that $\hat y_t = \hat \theta_t^\top x_t \approx \theta_t^\top x_t$. Thus, $y_t$ and $\hat y_t$ are likely to have similar sign. This entails that $e^{y_t\hat y_t}$ and thus $\eta_t$ are sufficiently large. So that, after some computations, we can upper-bound
% \[
%     (\theta_{t+1} - \htheta_t)^\top A_t (\theta_{t+1} - \htheta_t) - (\theta_t - \htheta_t)^\top A_{t-1} (\theta_t - \htheta_t) \leq \left(1+BR \right) \left(\frac{1}{2}e^{R^2/(2\lambda)} + 1\right) \frac{\eta_t}{2} g_t^\top A_t^{-1} g_t \,.
% \]
% Combining the two cases $z_t = y_t$ and $z_t = -y_t$ together, we get
% \[
%     \hat R_n(\theta) \leq \left(1+BR \right) \left(\frac{1}{2}e^{R^2/(2\lambda)} + 1\right)  \sum_{t=1}^n \frac{\eta_t}{2} g_t^\top A_t^{-1} g_t \,,
% \]
% which leaves us with a telescoping sum that finally provides the final regret upper-bound of the theorem. }
%
Using the definition of $\htheta_t$, after some computations, we can upper-bound
\[
    (\theta_{t+1} - \htheta_t)^\top A_t (\theta_{t+1} - \htheta_t) - (\theta_t - \htheta_t)^\top A_{t-1} (\theta_t - \htheta_t) \leq - \frac{1}{2} g_t^\top A_t^{-1} g_t^{-y_t} .
\]
Note that either $g_t$ or $g_t^{-y_t}$ is small. More precisely, if $\eta_t$ is exponentially small then this is also the case for $g_t^{-y_t}$ which is key to avoid the exponential constant. It should be put in comparison with the bound $g_t^\top  A_t^{-1} g_t$ that one would have obtained with the FTRL algorithm.
More precisely, we have the following relation $g_t^{-y_t} = -(1+BR)\eta_t g_t$ which leads to
\[
    \hat R_n(\theta) \leq (1+BR) \sum_{t=1}^n \frac{\eta_t}{2} g_t^\top A_t^{-1} g_t .
\]
This leaves us with a telescoping sum that finally provides the final regret upper-bound of the theorem. 

\end{sketchproof}

\subsection{Efficient Implementation}
\label{sec:implementation}

In this section, we show how to compute incrementally the proposed forecaster $\htheta_t$, defined in \eqref{eq:def_estimator}. First, we defined the sufficient statistics used by \namealgo as
\begin{equation}\label{eq:def-At-bt}
    A_t = \lambda I + \frac{1}{2}\sum_{s=1}^t \eta_s\,g_s g_s^\top, \quad b_t = \frac{1}{2}\sum_{s=1}^t (\eta_s g_s^\top \htheta_s - 1) g_s. 
\end{equation}
In the next lemma we characterize also $\htheta_t$ in terms $A_{t-1}, b_{t-1}, x_t$. 
\begin{lemma}[Characterizing $\htheta_t$ given $A_{t-1}, b_{t-1}, \theta_t, x_t$]\label{thm:dec-htheta-alg}\label{lm:dec-htheta-alg}
Using the notation above define
$$
W_t = L_{t-1}^{-1} (b_{t-1}, x_t) \in \R^{d\times 2}\,,
$$ 
where $L_{t-1}$ is the Cholesky decomposition of $A_{t-1}$, i.e. the lower triangular matrix satisfying $A_{t-1} = L_{t-1} L_{t-1}^\top$ and $\omega_t \in \R^2$ is the solution of the following problem
\begin{equation}\label{eq:small-problem}
    \omega_t = \argmin{\omega \in \R^{p_t}} \Omega_t(\omega), \quad  \Omega_t(\omega) = \|\omega\|^2 - 2u_t^\top \omega + \log(1+e^{-v_t^\top\omega}) + \log(1+e^{v_t^\top\omega}),
\end{equation}
    where $p_t \in \{1,2\}$ is the rank of the matrix $W_t$, $u_t = \Sigma_t^{1/2} U^\top e_1, v_t = \Sigma_t^{1/2} U^\top e_2$ with $\{U_t, \Sigma_t\}$ corresponding to the economic eigenvalue decomposition\footnote{I.e., $W_t^\top W_t = U_t \Sigma_t U_t^\top$ with $U_t \in \R^{2\times p_t}$ with $p_t$ the rank of $W_t^\top W_t$, such that $U_t^\top U_t = I$ and $\Sigma_t \in \R^{p_t \times p_t}$ is diagonal and positive.} of $W_t^\top W_t$ and $e_1= (1, 0), e_2 = (0, 1)$. Then 
\begin{equation}
    \htheta_t = L_{t-1}^{-\top} W_t U_t \Sigma_t^{-1/2}\omega_t.
\end{equation}
\end{lemma}
Computing $\htheta_t$ given $A_{t-1}, b_{t-1}, x_t$ therefore boils down to solving the two dimensional optimization problem in \eqref{eq:small-problem}, for which we can use gradient descent, since $\Omega_t$ is smooth strongly convex with a small condition number depending only on $R^2/\lambda$, as proven in the next lemma.

\begin{algorithm}[t]
\caption{\namealgo descriptive version (see Algorithm~\ref{algorithm-detail}, Appendix~\ref{app:implementation} for a detailed version)}
\label{algorithm-short}
%$ $\\
Parameters $\lambda, T, n$, constants $B, R$\\
    initialize $L_0 = \lambda^{1/2} I, b_0 = 0, \theta_0 = 0$\\
    \For{$t = 1,...,n$}{
    receive $x_t$ \\
    compute $W_t, U_t, \Sigma_t, u_t, v_t$ using $L_{t-1}, b_{t-1}, x_{t}, \theta_t$, as specified in Lemma~\ref{lm:dec-htheta-alg}\\
    compute $\omega^T_t$ using $T$ steps of gradient descent using $u_t, v_t$ as specified in Lemma~\ref{lm:solve-small-prob}\\
    compute $\tilde\theta_t = L_t^{-\top} W_t U_t \Sigma_t^{-1/2}\omega^T_t$ and predicts $\hat y_t = \tilde \theta_t^{\top} x_t$ \\
    receive $y_t$ \\
    compute $g_t$ using $\tilde\theta_t, y_t, x_t$\\
    compute $L_t$ via rank 1 Cholesky update of $L_{t-1}$ with vector $\sqrt{\eta_t/2} g_t$\\
    compute $b_t = b_{t-1} + (\eta_t g_{t}^\top \tilde\theta_{t} - 1)g_{t}$
    }
\end{algorithm}

\begin{lemma}\label{lm:solve-small-prob}
Let $\epsilon, \gamma > 0, T \in \mathbb{N}$, let $\omega_t$ be the solution of \eqref{eq:small-problem} and let $\omega^T_t$ be defined recursively as
$$\omega^i_t = \omega^{i-1}_t - \gamma \nabla \Omega_t(\omega^{i-1}_t), \quad \forall i \in \{1,\dots, T\}.$$
Then $\|\omega^T_t - \omega_t\| \leq \epsilon$, when $\omega^0_t = 0$ and $\gamma, T$ are chosen as follows
$$\gamma = \frac{\lambda}{4\lambda+R^2}, \quad T \geq \left(4+\frac{R^2}{\lambda}\right)\log\frac{Rt}{\epsilon\sqrt{\lambda}}.$$
\end{lemma}
The efficient sequential implementation of $\hat \theta_t$ reported in~\eqref{algorithm-short} is obtained by combining the different steps given by: the characterization $\hat \theta_t$ (Lemma~\ref{lm:dec-htheta-alg}); the efficient solution of~\eqref{eq:small-problem} (Lemma~\ref{lm:solve-small-prob}); and the fact that $L_t$ can efficiently be updated online by doing a Cholesky rank 1 update. 
More details on the algorithm are provided in Algorithm~\ref{algorithm-detail} in Appendix~\ref{app:implementation}. The total computational cost is of order $O(nd^2 + n \log n)$ as proven in the next theorem. The proof relies on the facts that rank 1 Cholesky updates cost $O(d^2)$ and that the cost of $w = L^{-1} v$ with $L \in \R^{d \times d}$ triangular invertible and $v,w\in \R^d$ (i.e. the solution of a triangular linear system $L w = v$)  is $O(d^2)$ \cite{golub2012matrix}.
\begin{theorem}[Efficient implementation]
\label{thm:comp-compl}
Let $T, n \in \mathbb{N}$ and $\tilde{\theta}_t$ be the solution of Algorithm~\ref{algorithm-detail} at step $t$, with hyperparameter $T$. Choosing $T = \left\lceil (4 + \frac{R^2}{\lambda}) \log \left( \frac{3n^2 R^2}{\lambda}(\frac{nR^2}{8\lambda} + B)\right)\right\rceil$ leads to a regret $\tilde{R}_n(\theta)$ for the forecaster $(\tilde{\theta}_t)_{t=1}^n$ bounded by
$$\tilde{R}_n(\theta) \le \lambda \|\theta\|^2 + d(1 + BR) \log \left(1 + \frac{nR^2}{8d(1+BR)\lambda} \right) + 1.$$
Moreover, Algorithm~\ref{algorithm-detail} has a total computational complexity
$$O\left(nd^2 + n \frac{R^2}{\lambda}\log\left[\frac{R n}{\lambda} + B\right]\right).$$
\end{theorem}
To conclude, note that, when $\lambda = \frac{1}{B^2}$, the total computational complexity of Algorithm~\ref{algorithm-detail} is $O(nd^2 + n \log n)$.

\section{Key ideas of the analysis}
\label{sec:technical_contrib}

In this section, we present more in details the two main ideas of our analysis. We believe that they might be of independent technical interest for future work. 

\subsection{Quadratic approximations with adaptive curvature}
\label{sec:quadratic_approx}

The main historical approach to prove logarithmic regret for online logistic regression is based on the observation that the logistic losses $\ell_t:\theta \mapsto \ell(\theta^\top x_t,y_t)$ are $\alpha$-exp-concave for some fixed exp-concavity parameter $\alpha >0$. In other words, for all $t\geq 1$, the functions $\theta \mapsto \exp\big(-\alpha \ell_t(\theta)\big)$ are convex.  From \cite[Lemma 4.2]{hazan2016introduction}, $\alpha$-exp-concavity implies in particular that for all $\theta, \hat \theta_t \in \mathcal{B}(\R^d,B)$ 
\begin{equation}
    \ell_t(\theta) \ge \ell_t(\hat \theta_t) + \nabla \ell_t(\hat \theta_t)^\top (\theta - \hat \theta_t) + \frac{\eta}{2} (\theta - \hat \theta_t)^\top \nabla \ell_t(\hat \theta_t) \nabla \ell_t(\hat \theta_t)^\top (\theta - \hat \theta_t)
    \label{eq:exp-concavity}
\end{equation}
where $\eta  \le \frac{1}{2} \min \{\frac{1}{4GB},\alpha\}$, where $G$ is an upper-bound of the $\ell_2$-norm of the gradients. We refer to $\eta$ as the \emph{curvature} constant. The above inequality provides a quadratic lower approximation of the logistic loss. It plays a crucial role in the analysis of ONS \cite{hazan2007logarithmic} to provide a logarithmic regret upper-bound of order $\smash{O(\frac{1}{\eta}d\log(n))}$. 
We can note that in this inequality, $\eta$ is fixed for all $t\geq 1$ and independent of $\theta$ and $\htheta_t$.  
However, for the logistic loss, the best exp-concavity constant $\alpha >0$ is of
order $e^{-BR}$ which leads to an undesirable exponential multiplicative constant. 

\medskip
Our idea is to replace the worst-case fixed $\eta > 0$ with a data adaptive constant $\eta_t$. To do so, we first remark that at time $t\geq 1$, the curvature constant is bad (i.e., of order $e^{-BR}$) when the prediction $\hat y_t = \hat \theta_t^\top x_t$ of the algorithm was significantly wrong. That is, when $y_t \hat y_t \approx - BR$. In contrast, if the algorithm predicted well the sign of the  next outcome, i.e., if $y_t \hat y_t \geq 0$ then Inequality~\eqref{eq:exp-concavity} holds with a much larger curvature constant greater than $(1+BR)^{-1}$. Based on this high-level idea, we could prove Inequality~\eqref{eq:exp-concavity} by replacing the fixed curvature $\eta >0$ with 
\begin{equation}
    \label{eq:etat}
    \eta_t = \frac{e^{y_t \hat y_t}}{1+BR}\,.
\end{equation}
The latter inequality yielded to our choice of surrogate quadratic approximations $\hat \ell_t$ defined in Equation~\eqref{eq:surrogate_loss_def}. 
This adaptive quadratic lower-approximation of the logistic loss is a direct consequence of the following technical lemma applied with $a = y_t \theta^\top x_t$, $b = y_t \hat y_t$, and $C = BR$. 

\begin{lemma}
\label{lem:quad_approx_lower} Let $C > 0$ and  $f: x \in \R \mapsto \log(1 + e^{-x})$. Then, for all $a \in [-C,C]$ and $ b \in \R$,
\[ 
    f(a) \ge f(b) + f'(b) (a - b) + \frac{e^b}{2(1+C)} f'(b)^2 (a - b)^2  \,.
\]
\end{lemma}
The proof is postponed to the supplementary material (see Appendix~\ref{app:lemmas}).

\subsection{Improper regularization}

The other key ingredient of our analysis is to ensure that only the rounds where the curvature $\eta_t$~\eqref{eq:etat} are  large matter in the analysis. This is the role of our new improper regularization added in the definition~\eqref{eq:def_estimator} of $\hat \theta_t$. The underlying idea is to add the possible next losses $\ell(\theta^\top x_t,1)$ and $\ell(\theta^\top x_t,-1)$ to the minimization problem solved by \namealgo (see~\eqref{eq:def_estimator}). 

We explain now the high-level idea why this regularization  helps when $\eta_t$ is small.
We need to distinguish two cases. 
On the one hand, if the prediction is good, i.e., $\hat y_t$ and $y_t$ have same signs. Then, $\eta_t \propto \exp(y_t \hat y_t)$ is large and since the prediction is already good. Thus, the regularization does not hurt much.
On the other hand, when $\hat y_t$ and $y_t$ have opposite signs, the curvature parameter may be exponentially small. But, then the addition of $\ell_t(\theta^\top x_t, y_t)$ greatly improves the predictions of the algorithm in these cases, because the data point $(x_t,y_t)$ was already included in the history when optimizing $\hat \theta_t$ in~\eqref{eq:regret_def}. Moreover, the  addition of the the wrong output $-y_t$ does not impact much the prediction since in $\hat y_t$ we have
\[
    \ell(\hat y_t,-y_t) = \log\big(1+e^{y_t \hat y_t}\big) = \log\big(1+(1+BR)\eta_t\big) 
\]
which is small whenever $\eta_t$ is small.

\section{Extensions}

\subsection{Non-parametric setting}
\label{sec:nonparametric}

For the sake of simplicity, the analysis of the present paper was only carried out for finite dimensional logistic regression in $\R^d$. Yet, most of the results remain valid for Reproducing Kernel Hilbert Spaces (RKHS) $\mathcal{H}$ (see~\cite{aronszajn1950theory} for details on RKHS). Then, Theorem \ref{thm:main_theorem} holds by replacing the finite dimension $d\geq 1$ with the effective dimension 
\[
\deff(\lambda) = \operatorname{Tr}(K_{nn}(K_{nn} + \lambda I)^{-1}) \,,
\]
where the input matrix $K_{nn}$ is defined as $(K_{nn})_{i,j} = x_i^\top x_j$. The regret is then of order $O(B\deff(B\lambda) + \lambda B^2)$. Note that the effective dimension is always upper-bounded by $\deff(\lambda) \leq n/\lambda$, providing in the worst case, the regret upper-bound of order $O(B\sqrt{n})$ for well-chosen $\lambda$. Under the capacity condition, which is a classical assumption for kernels (see \cite{marteau2019beyond} for instance), better bounds on the effective dimension are provided which yield to faster regret rates. 

In the case of RKHS, using standard kernel trick, the total computational complexity of the algorithm is then $O(n^3)$. The latter might be however prohibitive in large dimension. An interesting research direction is to investigate whether we can apply standard  approximation techniques such as random features or Nyström projection similarly to what \cite{calandriello2017efficient} and \cite{jezequel2019efficient} did for exp-concave and square loss respectively. In particular, what is the trade-off between computational complexity and regret and what is the lowest complexity that still allows optimal regret?

\subsection{Online-to-batch conversion}

Even in the batch statistical setting, the lower-bound of \cite{hazan2014logistic} holds for proper algorithms and few improper algorithms where introduced to avoid the statistical constant $O(e^B)$. Using the standard online-to-batch conversion \cite{helmbold1995weak}, similarly to the algorithm of \cite{foster2018logistic}, our algorithm also provides an estimator with bounded excess risk in expectation. To do so, one can sample an index $\tau$ uniformly in $\{1,\dots,n\}$ and define the estimator $\bar f_n$ defined for all $x \in \cX$ by
\begin{equation}
    \label{eq:batch_estimator}
   \bar f_n(x) = \hat f_\tau(x) \quad  \text{with} \quad  \hat f_t(x) = \hat \theta_t(x)^\top x, \qquad  1\leq t\leq T\,,
\end{equation}
where $\hat \theta_t(x)$ is the solution of the minimization problem~\eqref{eq:def_estimator} by substituting $x_t$ with the new input $x \in \cX$. It is worth pointing out that $\bar f_n(x) \neq \hat \theta_t^\top x$ is a non-linear function in $x$ and is thus improper. The following corollary controls the excess-risk of $\bar f_t$ in expectation. Its proof is standard but short and we recall it for the sake of completeness.

\begin{corollary}[Online-to-batch conversion]
\label{cor:batch} 
Let $n,d\geq 1$ and $B,R>0$. Let $\nu$ be an unknown distribution over $\cB(\R^d,R) \times \{-1,1\}$ and $D_n = \big\{(x_i,y_i)\big\}_{1\leq i\leq n}$ be i.i.d. sampled from $\nu$. Then, the estimator $\bar f_n$ defined in Equation~\eqref{eq:batch_estimator} with $\lambda = 1/B^2$ satisfies
\[
    \E\big[\ell(\bar f_n(X),Y)\big] - \inf_{\theta \in \mathcal{B}(\R^d,B)} \E\big[\ell(f(X),Y)\big] \leq  \frac{1}{n} \left[d(1 + BR)  \log \left(1 + \frac{nB^2R^2}{8d(1+BR)}\right) + 1 \right] \,,
\]
where $(X,Y) \sim \nu$ and the expectations are taken over $(X,Y)$, $D_n$ and $\tau$. 
\end{corollary}

\begin{proof}
Let us denote by $R_n^b$ the upper-bound on the regret in the right-hand side of Equation~\eqref{eq:regret_bound}. Then,
\begin{eqnarray*}
\E  \big[\ell(\bar f_n(X),Y)\big]  & = & \E\big[\ell( \hat f_\tau(X),Y)\big] \  =\  \E\Big[\frac{1}{n} \sum_{t=1}^n  \ell(\hat f_t(X),Y)\Big]\  \stackrel{(*)}{=}\  \E\Big [\frac{1}{n} \sum_{t=1}^n  \ell\big(\hat f_t(x_t),y_t\big)\Big] \\
& \leq &  \E \Bigg[\frac{1}{n} \sum_{t=1}^n \ell(\theta^\top x_t, y_t) \Big] + \frac{R_n^b}{n} \ \stackrel{(*)}{=} \ \E \bigg[\frac{1}{n} \sum_{t=1}^n \ell(\theta^\top X, Y) \bigg] + \frac{R_n^b}{n}\,,
\end{eqnarray*}
where the equalities $(*)$ are because $(X,Y)$ and $(x_t,y_t)$ follow the same distribution and because $\hat f_t$ and $\theta$ are independent of $(x_t,y_t)$ by definition.
\end{proof}

Apart from \cite{foster2018logistic}, which is non-practical and also based on an online-to-batch conversion, we are only aware of the works of \cite{mourtada2019improper} and \cite{marteau2019beyond} that improve the exponential constant $O(e^B)$ in the statistical setting. \cite{marteau2019beyond} make additional assumptions on the data distribution (self-concordance, well-specified model, capacity and source conditions). Their framework is hardly comparable to ours with constants that may be arbitrarily large in our setting. In contrast, the recent work of  \cite{mourtada2019improper} do provide an improper estimator that satisfies a result very similar to Corollary~\ref{cor:batch} with an expected bound on the excess risk of order
$O(d+B^2R^2)$. Our upper-bound is slightly worse with an additional multiplicative factor $BR \log (BRn)$.  The $\log n$ is due to the online setting in which it is optimal, see for instance the lower-bound of \cite{vovk2001competitive}. 
Their estimator is based on an empirical regularized risk minimization (with the original losses) with an additional improper regularization using virtual data. 
They do not analyze the computational complexity but we believe it to be similar to ours. To conclude the comparison, note that in contrast to ours, their analysis relies on the self-concordance property of the logistic loss in contrast to ours.

% we can leverage on the use of quadratic approximations to compute more efficiently the estimator for any $x$ (see section \ref{sec:implementation}). 

\section{Simulations}

This section illustrates our theoretical results with synthetic experiments and compare the performance of three algorithms: FTRL with $\ell_2$-regularization and $\lambda = 1$, \namealgo, and the one of~\cite{foster2018logistic}.
We sample the data points $(x_t,y_t) \in \R \times \{-1,1\}$ according to the  \textit{adversarial} distributions designed by \cite{hazan2014logistic} to prove the exponential lower bound for proper algorithms. We consider only the case $d = 1$ because the lower bound for proper algorithms already applies and the algorithm of \cite{foster2018logistic} is practical in this case.  Let $n \geq 1$, $B=\log(n)$, $\chi \in \{-1,1\}$ and $\epsilon = 0.01$, the data $(x_t,y_t)_{1\leq t\leq n}$ are i.i.d. generated according to  
\begin{equation*}
    (x_t,y_t) = \begin{cases}
               (1 - \frac{\sqrt{\epsilon}}{2B}, 1) & \text{w.p. } \frac{\sqrt{\epsilon}}{2B} + \chi \frac{\epsilon}{B} \\
               (\frac{\sqrt{\epsilon}}{B}, -1) & \text{otherwise}
           \end{cases} \,.
\end{equation*}
The experiment is averaged over 10 simulations  for $\chi = -1$ and 10 others for $\chi = 1$. We plot in Figure~\ref{fig:rates_plot} the worst of these two average regrets obtained by each algorithm according to the value of $n$. The lower-bound of \cite{hazan2014logistic} ensures that any proper algorithm has at least a regret of order $\Omega(n^{1/3})$ for these data. As expected, the regret of FTRL is polynomial in $n$ (linear slope in log-log scale) while the ones of \namealgo and the algorithm of \cite{foster2018logistic} are poly-logarithmic.

\begin{figure}
    \centering
    \includegraphics[width=.6\textwidth]{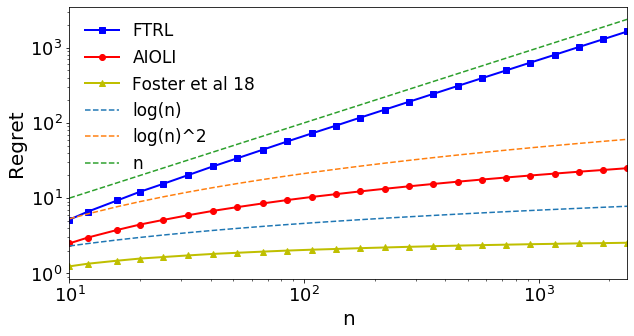}
    \caption{Regret of each algorithm according  to number of sample $n$ in log-log scale.}
    \label{fig:rates_plot}
\end{figure}

\section{Conclusion and future work}

To sum up, we designed a new efficient improper algorithm for online logistic regression. The latter only suffers logarithmic regret with much improved complexity compared to other existing methods. Some interesting questions are still remaining and left for future work.

\medskip 
Our online-to-batch procedure only provides upper-bounds in expectation. Obtaining high-probability bounds is more challenging and universal conversion methods such as the one of \cite{mehta2016fast} may not work for improper procedures. 

Another interesting direction for future research is the extension to multi-class classification. Our analysis strongly relies on binary outputs to produce the improper regularization and the extension to multi-class is not straightforward. The next step would then be to extend the results to other settings considered by \cite{foster2018logistic} such as  bandit multi-class learning or online multi-class boosting. More generally, it would be interesting to study what are the class of functions where adaptive \textit{curvature} parameters and improper learning yield to improved guarantees. 

Finally, as shown in Section~\ref{sec:nonparametric}, \namealgo may be applied to nonparametric logistic regression in RKHS. However, without any approximation schemes, the computational complexity may become prohibitive of order $O(n^3)$. Therefore, a possible line of research would be to study how much the performance of our algorithm would be affected by standard approximations techniques as Nyström projections or random features. 

\section*{Acknowledgements}

This work was funded in part by the French government under management of Agence Nationale de la Recherche as part of the "Investissements d’avenir" program, reference ANR-19-P3IA-0001 (PRAIRIE 3IA Institute).

\bibliography{biblio}
\bibliographystyle{apalike}

\newpage
\appendix

\section{Notation} \label{app:notation}
In this section, we recall and define useful notations that will be used all along the proofs. At each round $t \ge 1$, we recall that the forecaster is given an input $x_t \in \cX \subset \cB(\R^d,R)$; chooses a prediction $\hat \theta_t \in \R^d$; forms the prediction $\hat y_t = \hat \theta_t^\top x_t$; and observes the outcome $y_t \in \{-1,1\}$. The loss of a parameter $\theta \in \R^d$ at time $t\geq 1$ is measured by $\ell_t(\theta) = \ell(\theta^\top x_t, y_t) = \log(1+e^{-y_t \theta^\top x_t})$.

\medskip\noindent
We also define for all $t\geq 1$, all $\theta \in \R^d$ and $y \in \{-1,1\}$:
\begin{itemize}[label={-},parsep=2pt]
    \item the loss suffered by $\theta$ if the outcome was $y$: \qquad $\ell_t^y(\theta) = \log(1 + e^{-y \theta^\top x_t})$
    \item the gradient of the loss in $\hat \theta_t$ if the outcome was $y$:  \qquad $g_t^y = \nabla \ell_t^y(\htheta_t)$
    \item the curvature if the outcome was $y$: \qquad $\eta_t^y = \frac{e^{y \htheta_t^\top x_t}}{1+BR}$
    \item the quadratic surrogate losses if the outcome was $y$: \\
    \hspace*{1cm}$\hat \ell_t^y(\theta) = \ell_t^y(\htheta_t) + g_t^{y\top} (\theta - \htheta_t) + \frac{\eta_t^y}{2} (\theta - \htheta_t)^\top g_t^y g_t^{y\top} (\theta - \htheta_t)$
    \item the corresponding loss, surrogate loss, gradient, and curvature for the true outcome $\hat y_t$: \\
    \hspace*{1cm} $\ell_t = \ell_t^{y_t}$, \quad  $\hat \ell_t = \hat \ell_t^{y_t}$,\quad  $g_t = g_t^{y_t}$,\quad  $\eta_t = \eta_t^{y_t}$
    \item the regularized cumulative loss and cumulative surrogate loss respectively:  \\
    \hspace*{1cm} $L_t(\theta) = \sum\limits_{s=1}^t \ell_s(\theta) + \lambda \|\theta\|^2$,\quad  $\hat L_t(\theta) = \sum\limits_{s=1}^t \hat \ell_s(\theta) + \lambda \|\theta\|^2$
\end{itemize}
With these notations, we defined $\theta_{t}$ and $\bar \theta_{t}$ as:
\begin{equation}
\label{eq:recall_def_theta}
\theta_{t} = \argmin{\theta \in \R^d} \hat L_{t-1}(\theta)  \,,\quad  \text{and}\quad \bar \theta_t = \argmin{\theta \in \R^d} \left\{ \hat L_{t-1}(\theta) + \ell_t^{1}(\theta) + \ell_t^{-1}(\theta) \right\} \,.
\end{equation}

\section{Proof of the main theorem}

\label{app:proof_main_thm}

\begin{theorem}
\label{thm:main_theorem_approx}
Let $\epsilon, \lambda, R>0$ and  $d, n\geq 1$. Let $(x_1,y_1),...,(x_n,y_n) \in \cX \times \cY$ be an arbitrary sequence of observations.  Define $\bar \theta_t$ for $t
\geq 1$ as in Equation \eqref{eq:recall_def_theta} with regularization parameter $\lambda > 0$. Then, any estimator $\htheta_t$ which verifies for all $t \geq 1$, $\|\htheta_t - \bar \theta_t\| \le \epsilon$  satisfies the following upper-bound on the regret
\[ 
R_n(\theta) \le \lambda \|\theta\|^2 + d(1 + BR) \log \left( 1 + \frac{n R^2}{8d(1+BR)\lambda} \right) + 3nR\left( \frac{nR^2}{8\lambda} + B \right) \epsilon \,.
\]
\end{theorem}

\begin{proof}
Let $\theta \in \cB(\R^d,B)$. Let us first upper-bound the regret $R_n(\theta)$ by the regret using the surrogate losses.  Applying Lemma~\ref{lem:quad_approx_lower} with $a = y_t\theta^\top x_t \in [-BR,BR]$ and $b \in y_t \hat \theta_t^\top x_t \in \R$, we have for all $t\geq 1$:
\[
\ell_t(\theta) \ge \hat \ell_t(\theta).
\]
Together with $\ell_t(\hat \theta_t) = \hat \ell_t(\hat \theta_t)$, it yields that the regret on the true loss is upper-bounded by the regret on the quadratic approximations 
\begin{align}
\label{eq:R_bound_Rhat}
    R_n(\theta) = \sum\limits_{t=1}^n \ell_t(\htheta_t) - \sum\limits_{t=1}^n \ell_t(\theta) \le \sum\limits_{t=1}^n \hat \ell_t(\htheta_t) - \sum\limits_{t=1}^n \hat \ell_t(\theta) = \hat R_n(\theta) \,.
\end{align}
Now, we are left with analyzing a quadratic problem. We can thus follow in the main lines the proof of \cite{gaillard2018uniform} for online least squares.  By definition of $\theta_{n+1} = \argmin{\theta \in \R^d} \hat L_{n}(\theta)$,  $\hat L_n(\theta_{n+1}) \le \hat L_n(\theta)$, which can be written as
\begin{align*} \sum\limits_{t=1}^n \hat \ell_t(\theta_{n+1}) - \sum\limits_{t=1}^n \hat \ell_t(\theta) &\le \lambda \|\theta\|^2 - \lambda \|\theta_{n+1}\|^2 \,.
\end{align*}
Now, the regret can be upper-bounded as
\begin{align}
    \hat R_n(\theta) &\le \lambda \|\theta\|^2 + \sum\limits_{t=1}^n \hat \ell_t(\htheta_t) - \sum\limits_{t=1}^n \hat \ell_t(\theta_{n+1})  - \lambda \|\theta_{n+1}\|^2 \nonumber \\
    &= \lambda \|\theta\|^2 + \sum\limits_{t=1}^n \left[ \hat \ell_t(\htheta_t) + \hat L_{t-1}(\theta_t) - \hat L_t(\theta_{t+1}) \right] \,.
    \label{eq:quadratic_regret}
\end{align}
With a little bit of abuse, we call the terms inside the sum on the right the \emph{instant regrets}. Grouping the terms of same degrees in the quadratic approximation,
\begin{align*}
    \hat \ell_t(\theta) &= \ell_t(\htheta_t) + g_t^\top (\theta - \htheta_t) + \frac{\eta_t}{2} (\theta - \htheta_t) g_t g_t^\top (\theta - \htheta_t) \\
    &= \ell_t(\htheta_t) - g_t^\top \htheta_t + \frac{\eta_t}{2} \htheta_t^\top g_t g_t^\top \htheta_t + g_t^\top \theta - \eta_t \theta^\top g_t g_t^\top \htheta_t + \frac{\eta_t}{2} \theta^\top g_t g_t^\top \theta \\
    &= c_t^* - 2b_t^{* \top} \theta + \frac{\eta_t}{2} \theta^\top g_t g_t^\top \theta
\end{align*}
with $c_t^* = \ell_t(\htheta_t) - g_t^\top \htheta_t + \frac{\eta_t}{2} \htheta_t^\top g_t g_t^\top \htheta_t$ and $b_t^* = \frac{1}{2} \big(- g_t + \eta_t (\htheta_t^\top g_t) g_t \big)$. Similarly, we can write the cumulative loss as
\begin{equation}
\label{eq:L_hat_quadratic}
\hat L_t(\theta) = \underbrace{\sum\limits_{s=1}^t c_s^*}_{c_t} - 2  \bigg(\underbrace{ \sum\limits_{s=1}^t b_s^*}_{b_t}\bigg)^\top \theta + \theta^\top \bigg(\underbrace{ \lambda I + \sum\limits_{s=1}^t \frac{\eta_s}{2} g_s g_s^\top }_{A_t}\bigg) \theta \,.
\end{equation}
The minimum of this quadratic, reached in $\theta_{t+1} = A_t^{-1} b_t$, is 
\[
    \hat L_t(\theta_{t+1}) = c_t - 2 \underbrace{b_t^\top A_t^{-1}}_{\theta_{t+1}} A_t \theta_{t+1} + \theta_{t+1}^\top A_t \theta_{t+1} 
    = c_t - \theta_{t+1}^\top A_t \theta_{t+1} \,.
\]
We can write now the instant regret at $t$ as
\begin{align}
    \hat \ell_t(\htheta_t) + \hat L_{t-1}(\theta_t) - \hat L_t(\theta_{t+1})
    &= \hat \ell_t(\htheta_t) - c_t^* + \theta_{t+1}^\top A_t \theta_{t+1} - \theta_t^\top A_{t-1} \theta_t \nonumber \\
    &= g_t^\top \htheta_t - \frac{\eta_t}{2} \htheta_t^\top g_t g_t^\top \htheta_t + \theta_{t+1}^\top A_t \theta_{t+1} - \theta_t^\top A_{t-1} \theta_t \nonumber \\
    &= g_t^\top \htheta_t - \htheta_t^\top (A_t - A_{t-1}) \htheta_t   + \theta_{t+1}^\top A_t \theta_{t+1} - \theta_t^\top A_{t-1} \theta_t \,. \label{eq:instant_regret1}
\end{align}
The oracle $\theta_{t+1}$ minimizes the quadratic function $\hat L_t$ with Hessian $2 A_t$. Thus, performing one newton step from $\htheta_t$ gives
\begin{align}
\label{eq:rel_theta_1}
    \theta_{t+1} &= \hat \theta_t - \frac{1}{2}A_t^{-1} \nabla \hat L_t(\hat \theta_t) \nonumber\\
    &= \hat \theta_t - \frac{1}{2}A_t^{-1} \left[\nabla \hat L_{t-1}(\hat \theta_t) + g_t\right] \nonumber \\
    &= \hat \theta_t + \frac{1}{2} A_t^{-1} g_t^{-y_t} - \frac{1}{2}A_t^{-1} \left[\nabla \hat L_{t-1}(\hat \theta_t) + g_t + g_t^{-y_t}\right] 
\end{align}
Similarly, we have
\begin{align}
\label{eq:rel_theta_2}
\theta_t &= \htheta_t - \frac{1}{2} A_{t-1}^{-1} \nabla \hat L_{t-1}(\htheta_t) \nonumber \\
&= \htheta_t + \frac{1}{2} A_{t-1}^{-1} (g_t + g_t^{-y_t}) - \frac{1}{2} A_{t-1}^{-1} \left[\nabla \hat L_{t-1}(\hat \theta_t) + g_t + g_t^{-y_t} \right] .
\end{align}
Reorganizing the terms in the two previous equations leads to
\begin{align*}
g_t &= 2 \left[ A_t \htheta_t - A_t \theta_{t+1} - A_{t-1} \htheta_t + A_{t-1} \theta_t \right].
\end{align*}
Substituting in the instant regret~\eqref{eq:instant_regret1}, this entails
\begin{align}
    \hat \ell_t(\htheta_t) & + \hat L_{t-1}(\theta_t) - \hat L_{t}(\theta_{t+1}) = 2 \htheta_t^\top A_t \htheta_t - 2 \theta_{t+1}^\top A_t \htheta_t  - 2 \htheta_t^\top A_{t-1} \htheta_t + 2 \theta_t^\top A_{t-1} \htheta_t \nonumber  \\
    &\quad\quad - \htheta_t^\top A_t \htheta_t +  \htheta_t^\top A_{t-1} \htheta_t + \theta_{t+1}^\top A_t \theta_{t+1} - \theta_t^\top A_{t-1} \theta_t \nonumber \\
    &= \theta_{t+1}^\top A_t \theta_{t+1}- 2 \theta_{t+1}^\top A_t \htheta_t + \htheta_t^\top A_t \htheta_t - \theta_t^\top A_{t-1} \theta_t + 2 \theta_t^\top A_{t-1} \htheta_t - \htheta_t^\top A_{t-1} \htheta_t \nonumber \\
    &= (\theta_{t+1} - \htheta_t)^\top A_t (\theta_{t+1} - \htheta_t) - (\theta_t - \htheta_t)^\top A_{t-1} (\theta_t - \htheta_t)
    \label{eq:instant_regret2}
\end{align}
Rewriting equations \eqref{eq:rel_theta_1} and \eqref{eq:rel_theta_2}, we have
\begin{align*}
    2 A_t (\theta_{t+1} - \htheta_t) &= g_t^{-y_t} - \delta_t \\
    2 A_{t-1} (\theta_t - \htheta_t) &= g_t^{-y_t} + g_t - \delta_t
\end{align*}
with $\delta_t = \nabla \hat L_{t-1}(\hat \theta_t) + g_t + g_t^{-y_t}$. \\
Subtracting the first equation to the second, we can write the instant regret as a variance term and an optimization error term,
\begin{equation}
\label{eq:variance_opt_decomposition}
    \hat \ell_t(\htheta_t) + \hat L_{t-1}(\theta_t) - \hat L_{t}(\theta_{t+1})
    = Z_t + \Omega_t
\end{equation}
where
\begin{align*}
    Z_t &= \frac{1}{4} g_t^{-y_t \top} A_t^{-1} g_t^{-y_t} - \frac{1}{4} (g_t + g_t^{-y_t}) A_{t-1}^{-1} (g_t + g_t^{-y_t}) 
\end{align*}
and
\begin{align*}
    \Omega_t
    &= \frac{1}{4} \left[-2 g_t^{-y_t} A_t^{-1} \delta_t + \delta_t A_t^{-1} \delta _t + 2 (g_t + g_t^{-y_t}) A_{t-1}^{-1} \delta_t - \delta_t A_{t-1}^{-1} \delta_t \right] \,.
\end{align*}

\subsection{Upper-bound of the variance term $Z_t$}

Let us focus on bounding the term $Z_t$. Developing the terms and using the fact that $A_{t-1} \le A_t$, we have
\begin{align*}
    Z_t
    &= \frac{1}{4} (g_t + g_t^{-y_t}-g_t)^\top A_t^{-1} (g_t + g_t^{-y_t} - g_t) 
    - \frac{1}{4} (g_t + g_t^{-y_t})^\top A_{t-1}^{-1} (g_t + g_t^{-y_t}) \\
    &= \frac{1}{4} g_t^\top A_t^{-1} g_t - \frac{1}{2} g_t^{\top} A_t^{-1} (g_t + g_t^{-y_t}) + \frac{1}{4} (g_t+g_t^{-y_t})^\top A_t^{-1} (g_t+g_t^{-y_t}) \\
    &\quad - \frac{1}{4} (g_t+g_t^{-y_t})^\top A_{t-1}^{-1} (g_t+g_t^{-y_t}) \\
    &\le \frac{1}{4} g_t^\top A_t^{-1} g_t - \frac{1}{2} g_t^{\top} A_t^{-1} (g_t + g_t^{-y_t}) \\
    &= -\frac{1}{4} g_t^\top A_t^{-1} g_t - \frac{1}{2} g_t^{\top} A_t^{-1} g_t^{-y_t} \\
    &\le - \frac{1}{2} g_t^{\top} A_t^{-1} g_t^{-y_t} \,.
\end{align*}
Using the definition of the logistic function, we can relate $g_t$ and $g_t^{-y_t}$,
\begin{equation}
    g_t^{-y_t} = \frac{y_t x_t}{1 + e^{-y_t \htheta_t^\top x_t}} = e^{y_t \htheta_t^\top x_t}\frac{y_t x_t}{1 + e^{y_t \htheta_t^\top x_t}} =  - (1+BR) \eta_t g_t \,,
    \label{eq:gtequality}
\end{equation}
which implies,
\begin{align*}
    Z_t
    &\le (1+BR) \frac{\eta_t}{2} g_t^\top A_t^{-1} g_t \,.
\end{align*}
%Then a direct consequence of Lemma \ref{lem:bound_eta} is $\eta_t \ge \frac{e^{-R^2/(2\lambda)}}{1+BR}$, which leads to
%\begin{align*}
%    \hat \ell_t(\htheta_t) + \hat L_{t-1}(\theta_t) - \hat L_t(\theta_{t+1})
%    &\le  \left[ \frac{1}{2}e^{R^2/(2\lambda)} + 1\right] (1+BR) \frac{\eta_t}{2} g_t^\top A_t^{-1} g_t \,.
%\end{align*}
Summing over $t=1,\dots,n$, the sum telescopes thanks to Lemma \ref{thm:telescop_sum}, we obtain
\begin{equation}
\sum_{t=1}^n Z_t \le (1+BR) \sum\limits_{k=1}^d \log \left( 1 + \frac{\lambda_k(C_n)}{\lambda} \right) \,,\label{eq:presquefini}
\end{equation}
where $C_n = \frac{1}{2} \sum_{t=1}^n \eta_t g_t g_t^\top$ and $\lambda_k(C_n)$ is the $k$ largest eigenvalue of $C_n$. 

Now to upper-bound the right-hand side we need to upper-bound the trace of $C_n$, which we do now.
Recalling that $g_t = - y_t x_t/\big(1+\exp(y_t \htheta_t^\top x_t)\big)$, we have 
\[
\frac{\eta_t}{2} g_t g_t^\top = \frac{1}{2(1+BR)} \frac{e^{y \htheta_t^\top x_t}}{(1 + e^{y_t \htheta_t^\top x_t})^2} x_t x_t^\top 
\le \frac{1}{8(1 + BR)} x_t x_t^\top\,,
\]
where for the inequality, we used that $x/(1+ x)^2\leq 1/4$ for $x\geq 0$. Therefore, $\operatorname{Tr}(C_n) = \sum_{k=1}^d \lambda_k(C_n) \leq nR^2/(8(1 + BR))$ for all $k\geq 1$. Now remark that the right-hand side of equation \ref{eq:presquefini} is maximized under the constraint $\sum_{k=1}^d \lambda_k(C_n) \leq nR^2/(8(1 + BR))$ when all the eigenvalues are equals \ie $\lambda_k(C_n) = nR^2/(8d(1 + BR))$ for all $1 \le k \le d$ which leads to
\begin{align}
    \sum_{t=1}^n Z_t &\le d(1+BR)\log \left( 1 + \frac{nR^2}{8d(1+BR)\lambda} \right)
\end{align}

\subsection{Upper-bound on the optimization error $\Omega_t$}

It remains to bound the the approximation term $\Omega_t$.
\begin{align*}
    \Omega_t
    &= \frac{1}{4} [2 (g_t + g_t^{-y_t}) A_{t-1}^{-1} \delta_t -2 g_t^{-y_t} A_t^{-1} \delta_t + \underbrace{\delta_t A_t^{-1} \delta _t - \delta_t A_{t-1}^{-1}\delta_t}_{\le 0} ] \\
    &\le \frac{1}{2\lambda}\left[ 2\|g_t^{-y_t}\| + \|g_t\| \right] \|\delta_t\| \\
    &\le \frac{3R}{2\lambda} \|\delta_t\| \,.
\end{align*}
The last inequality is due to $\|g_t^y\| \le R$ for all $\|x_t\| \le R$ and $y \in \{-1,1\}$.
By definition of $\bar \theta_t \in \argmin{\theta \in \R^d} \left\{ \hat L_{t-1}(\theta) + \ell_t(\theta) + \ell_t^{-y_t}(\theta) \right\}$, we have $\nabla \hat L_{t-1}(\bar \theta_{t}) + \nabla \ell_t(\bar \theta_t) + \nabla \ell_t^{-y_t}(\bar \theta_t) = 0$.
Note also that $g_t^y = \nabla \ell_t^{y}(\htheta_t)$ for all $y \in \{-1,1\}$.
So $\delta_t$ may be rewriten as
\[ \delta_t = \nabla \hat L_{t-1}(\htheta_{t}) + \nabla \ell_t(\htheta_t) + \nabla \ell_t^{-y_t}(\htheta_t) - \nabla \hat L_{t-1}(\bar \theta_{t}) - \nabla \ell_t(\bar \theta_t) - \nabla \ell_t^{-y_t}(\bar \theta_t)\]
Using that $\nabla \ell_t$ and $\nabla \hat \ell_t$ are $R^2/4$-Lipschitz (for $\nabla \hat \ell_t$ remark that $\|\nabla^2 \hat \ell_t(\theta)\| = \| \eta_t g_t g_t^\top \| \le R^2/4)$), we have
\begin{align}
    \|\delta_t\| &\le \left[\frac{(t+1)R^2}{4} + 2\lambda B \right] \|\hat \theta_t - \htheta_t\| \,.
\end{align}
Summing over $t$ leads to
\begin{equation}
\label{eq:bound_sum_omega}
    \sum_{t=1}^n \Omega_t
    \le 3nR \left( \frac{nR^2}{8\lambda} + B \right) \epsilon \,.
\end{equation}
\subsection{Conculsion of the proof}
Using inequalities \eqref{eq:R_bound_Rhat}, \eqref{eq:quadratic_regret} and \eqref{eq:variance_opt_decomposition}, we have
\begin{equation*}
    R_n(\theta) \le \hat R_n(\theta) \le \lambda \|\theta\|^2 + \sum_{t=1}^n Z_t + \sum_{t=1}^n \Omega_t
\end{equation*}
Finally, inequalities~\eqref{eq:presquefini} and \eqref{eq:bound_sum_omega} concludes the proof.
\end{proof}

\section{Lemmas}
\label{app:lemmas}

\begin{proof}\textbf{of Lemma \ref{lem:quad_approx_lower}.}
Let $C >0$. First, note that for all $x \in \R$, $f'(x) = -(1+\exp(x))^{-1}$. 
To prove Lemma~\ref{lem:quad_approx_lower}, we need to show that for $\alpha = (1+C)^{-1}$, we have for all $a \in [-C,C]$ and $b \in \R$
\[ 
\log(1 + e^{-a}) \ge \log(1 + e^{-b}) - \frac{1}{1 + e^{b}} (a - b) + \frac{\alpha}{2} \frac{e^{b}}{(1 + e^{b})^2} (a - b)^2 \,.
\]
To do so, we fix $b \in \R$ and we define the function $\xi$ as
\[
\xi(a) = \log(1 + e^{-a}) - \log(1 + e^{-b}) + \frac{1}{1 + e^{b}} (a - b) - \frac{\alpha}{2} \frac{e^{b}}{(1 + e^{b})^2} (a - b)^2, \qquad -C\leq a\leq C
\]
It remains to show that $\xi$ is non-negative on $[-C,C]$.
Because $\xi(b) = 0$, it suffices to prove
\begin{equation}
\xi'(a) \begin{cases} 
            \le 0 & \text{for}\  a\leq b \\
            \geq 0 & \text{for}\  a \geq b
        \end{cases}
    \label{eq:xiprimesign}
\end{equation}
First, after some computation, differentiating $\xi$ leads to
\[ 
\xi'(a) = -\frac{1}{1 + e^a} + \frac{1}{1 + e^{b}} - \frac{\alpha e^{b}}{(1 + e^{b})^2} (a - b) \,,
\]
which can also be rewritten as
\[
    \xi'(a) = \frac{e^a - e^{b}}{(1 + e^a)(1 + e^{b})} - \alpha \frac{e^{b}}{(1 + e^{b})^2}(a-b) \,.
\]
Reorganizing the terms gives the following equation
\[
    (1+e^a)(1 + e^{b})e^{-b} \xi'(a) = e^{a - b} - 1 - \alpha \frac{1 + e^a}{1 + e^{b}}(a - b) \,.
\]
Therefore, \eqref{eq:xiprimesign} holds true as soon as 
\[ 
\alpha \le \frac{e^{a - b}-1}{a-b} \frac{1 + e^{b}}{1 + e^a} \,,
\]
with the convention $(e^0-1)/0 = 1$. The latter is satisfied by Lemma \ref{thm:ineq_indpt_b}, because $\alpha = (1+C)^{-1} \leq (1+|a|)^{-1}$ for all $a \in [-C,C]$.  
\end{proof}

\begin{lemma}
\label{thm:ineq_indpt_b}
For all $a, b \in \R$, 
\[
    \frac{1}{1+|a|} \leq \frac{e^{a - b}-1}{a-b} \frac{1 + e^{b}}{1 + e^a}  \,.
\]
\end{lemma}
\begin{proof}
Define the function $h:\R^2 \mapsto \R$ that corresponds to the right-hand side of the inequality
\[
    h(a,b) = \frac{e^{a - b}-1}{a-b} \frac{1 + e^{b}}{1 + e^a},  \qquad (a,b) \in \R^2\,.
\]
It is worth pointing out that even $h$ is normally not defined for $a = b$, setting $h(a,a) = 1$ makes it well defined and infinitely differentiable on $\R^2$. 

\medskip\noindent
Let $a \ge b$. Then  $(1 + e^{b})/(1 + e^a) \ge e^{b-a}$, which implies
\[
    h(a,b) \ge \frac{(e^{a- b}-1)e^{b-a}}{(a - b)} 
    \ge \frac{e^{b-a}-1}{b-a} 
    \geq \frac{1}{1 + a - b} \,,
\]
where the last inequality is because  $(e^x - 1)/x \ge (1-x)^{-1}$ for all $x \le 0$. 

\medskip\noindent
Otherwise, let $a \le b$.  Then $(1 + e^{b})/(1 + e^a) \ge 1$, which entails
\[
    h(a,b) \ge \frac{e^{a- b}-1}{(a - b)} \ge \frac{1}{1 + b - a}
\]

\medskip\noindent
Combining the two cases $a\leq b$ and $a\geq b$ together, we get
\begin{equation}
    h(a,b) \ge \frac{1}{1 + |a - b|} \,. \label{eq:naive_ineq_h}
\end{equation}

\medskip\noindent
Now we show that $\argmin{b \in \R} h(a,b)$ contains a value between $0$ and $a$, i.e., in  $[0,a]$ if $a \ge 0$ or in $[a,0]$ otherwise. Rewriting the function $h$ as follows,
\begin{align}
\label{eq:h_sym}
    h(a,b) &= \frac{1}{1+e^a}\frac{e^{a-b} - 1}{a-b} + \frac{1}{1+e^{-a}}\frac{e^{b-a} - 1}{b-a}\,,
\end{align}
it is clear that $h(a,b) = h(-a,-b)$. We can therefore suppose without loss of generality that $a$ is non-negative. Indeed, if $b^*(a) \in \argmin{b} h(a,b)$ then $-b^*(a) \in \argmin{b} h(-a,b)$.
A further look at Equation \eqref{eq:h_sym} shows also that $h$ is convex in its second argument by convexity of the function $x \mapsto (e^x - 1)/x$ and stability of convex functions by composition with affine transformations and non-negative weighted sum.

To finish the proof, we will show that the derivative of $\partial h(a,b)/\partial b$ is non-negative for $b\to a$ and non-positive for $b = 0$. Convexity  of $h$ in its second argument will then conclude. Using $a \ge 0$, some computations (omitted here) lead to
\begin{align*}
    \frac{\partial}{\partial b}h(a,b) &= \frac{(b-a + 1)e^{a-b} - 1 + (b-a-1)e^b + e^a}{(a-b)^2(1+e^a)} \,.
\end{align*}
Develloping the first terms of the exponential series gives
\[ 
(b-a + 1)e^{a-b} - 1 + (b-a-1)e^b + e^a = \frac{1}{2}(e^a-1)(a-b)^2 + o_{b\to a}((a-b)^2)\,.
\]
Therefore, 
\[
    \lim_{b \rightarrow a} \frac{\partial}{\partial b}h(a,b) =\frac{1}{2}\frac{e^a-1}{e^a+1} \  \ge \ 0 \,.
\]
We note also that if $a=0$ then $\frac{\partial}{\partial b}h(0,0)=0$.
Now, if $a > 0$, using $(2-x) e^x \le 2+x$ for all $x \ge 0$, we have
\[
    \frac{\partial}{\partial b}h(a,0) =
    (2 - a)e^a - (a+2) 
    \le 0
\]
By convexity of the function $b \mapsto h(a,b)$, we conclude that $\argmin{b \in \R} h(a,b) \in [0,a]$. Combined with Inequality~\eqref{eq:naive_ineq_h} concludes the proof of the lemma.
\end{proof}

The following Lemma is a standard result of online matrix theory (Lemma~11.11 of \cite{cesa2006prediction}).
\begin{lemma}
    \label{lem:linearalgebra}
    Let $V \in \R^{d \times d}$ be an invertible matrix, $u \in \R^d$ and $U = V - u u^\top$. Then,
    \[
        u^\top V^{-1} u = 1 - \frac{\det(U)}{\det(V)} \,.
    \]
\end{lemma}

\begin{lemma}
\label{thm:telescop_sum}
If $C_n = \sum\limits_{t=1}^n \frac{\eta_t}{2} g_t g_t^\top$ and $A_n = C_n + \lambda I$ then
\[ \sum\limits_{t=1}^n \frac{\eta_t}{2} g_t^\top A_t^{-1} g_t \le \sum\limits_{k=1}^d \log \left(1 + \frac{\lambda_k(C_n)}{\lambda} \right)\]
where $\lambda_k(C_n)$ is the $k$ largest eigenvalue of $C_n$.
\end{lemma}
\begin{proof}
Remarking that $A_t = A_{t-1} + \frac{\eta_t}{2} g_t g_t^\top$ and applying lemma \ref{lem:linearalgebra} we have
\[ \frac{\eta_t}{2} g_t^\top A_t^{-1} g_t = 1 - \frac{\operatorname{det}(A_{t-1}/\lambda)}{\operatorname{det}(A_t/\lambda)} \]
We use now that $1-u \le \log(1/u)$ for $u>0$ which yields
\[ \frac{\eta_t}{2} g_t^\top A_t^{-1} g_t \le\log \frac{\operatorname{det}(A_{t}/\lambda)}{\operatorname{det}(A_{t-1}/\lambda)} \]
Summing over $t = 1,...,n$, using $A_0 = \lambda I$ and $A_n = C_n + \lambda I$ with $C_n = \sum\limits_{t=1}^n \frac{\eta_t}{2} g_t g_t^\top$, we get
\begin{align*}
\sum\limits_{t=1}^n \frac{\eta_t}{2} g_t^\top A_t^{-1} g_t &\le \log \left( \operatorname{det} \left(I + \frac{C_n}{\lambda} \right)\right) \\
&= \sum\limits_{k=1}^d \log \left(1 + \frac{\lambda_k(C_n)}{\lambda} \right)
\end{align*}

\end{proof}

\section{Efficient implementation of \namealgo}
\label{app:implementation}

\begin{algorithm}
\caption{\namealgo detailed version. Here $\operatorname{eigen-dec}$ corresponds to economic-eigendecomposition of a symmetric matrix and $\operatorname{chol-update}$ to the rank 1 Cholesky update \cite{golub2012matrix}\label{algorithm-detail}}
    $L_0 = \lambda^{-1/2} I, \tilde b_0 = 0, \tilde \theta_0 = 0$\\
    \For{$t = 1,...,n$}{
    receives $x_t$ \\
    $W_t = L_t^{-1}(b_{t-1}, x_t),  ~~\{U_t, \Sigma_t\} = \operatorname{eigen-dec}(W_t^\top W_t), u_t = \Sigma_t^{1/2} U_t^\top e_1, ~~ v_t = \Sigma_t^{1/2} U_t^\top e_2$\\
    $\omega^0_t = 0$\\
    \For{$i = 1,...,T$}{
    $\omega^i_t = \omega^{i-1}_t - \frac{\lambda}{4\lambda+R^2} \left[2\omega^{i-1}_t -2 u_t - (1+e^{v_t^\top \omega^{i-1}_t})^{-1} v_t + (1+e^{-v_t^\top \omega^{i-1}_t})^{-1} v_t\right]$
    }
    $\tilde \theta_t = L_t^{-\top} (W_t (U_t \Sigma_t^{-1/2} \omega^T_t))$\\
    predict $\hat y_t = \tilde \theta_t^{\top} x_t$ \\
    receives $y_t$ \\
    $g_t = -(1 + e^{y_t \tilde\theta_t^\top x_t})^{-1} y_t x_t, \quad \eta_t = \frac{e^{y_t \hat\theta_t^\top x_t}}{1+BR}$\\
    $L_t = \operatorname{chol-update}(L_{t-1}, \sqrt{\eta_t/2} g_t), \quad b_t = b_{t-1} + (\eta_t g_{t}^\top \tilde\theta_{t} - 1)g_{t}$\\
    }
\end{algorithm}

\begin{proof}{\bf of Lemma~\ref{lm:dec-htheta-alg}}
Given the definition of $\htheta_t$, in \eqref{eq:def_estimator}, using the notation in Appendix~\ref{app:notation} and Eq.~\eqref{eq:L_hat_quadratic}
\begin{align*}
    \htheta_t &=~~ \argmin{\theta \in \R^d} ~\widehat{L}_{t-1}(\theta) + \log(1 + e^{-\theta_t^\top x_t}) + \log(1 + e^{\theta^\top x_t}) \\
    & = ~~\argmin{\theta \in \R^d} ~\theta^\top A_{t-1} \theta -2 \theta^\top b_{t-1}  + \log(1 + \ch(\theta^\top x_t) ).
\end{align*}
Since $A_{t-1}$ is invertible by construction and $A_{t-1} = L_{t-1} L_{t-1}^\top$, where $L_{t-1}$ is lower triangular and the unique Cholesky decomposition of $A_{t-1}$ \cite{golub2012matrix}, we can define the following equivalent problem, by the substitution $r = L_{t-1}^\top \theta$
$$ r_t ~~=~~ \argmin{r \in \R^d} ~r^\top r - 2 r^\top \tilde u_t + \log(1 + \ch(r^\top \tilde v_t)) , \quad \tilde u_t = L_{t-1}^{-1} b_{t-1}, \quad \tilde v_t = L_{t-1}^{-1} x_t,$$
and in particular $r_t = L_t^\top \htheta_t$. Now note that  any $r \in \R^d$ can be always written as $r = W_t q + \mu$ with $W_t = (\tilde u_t, \tilde v_t) \in \R^{d \times 2}$ for some $q \in \R^2$ and $\mu \in (\operatorname{span} W_t)^\bot$, then $r_t = W_t q_t + \mu_t$ for $q_t, \mu_t$ defined as 
\begin{align*}
(q_t, \mu_t) &= \argmin{q \in \R^2, \mu \in (\operatorname{span} W_t)^\bot} \|W_t q+ \mu\|^2 - 2\tilde u_t^\top(W_t q + \mu)  + \log(1 + \ch(\tilde v_t^\top (W_t q + \mu)))\\
&= \argmin{q \in \R^2, \mu \in (\operatorname{span} W_t)^\bot} \|W_t q\|^2 + \|\mu\|^2 - 2\tilde u_t^\top W_t q  + \log(1 + \ch(\tilde v_t^\top W_t q)),
\end{align*}
where in the last inequality we use the fact that $\tilde u_t = W_t e_1$, $\tilde v_t = W_t e_2$ and $W_t^\top \mu = 0$, by construction. Now the solution of the problem above is given by $\mu_t = 0$ and $q_t$ as
$$ q_t ~~=~~  \argmin{q \in \R^2} ~\|W_t q\|^2 - 2\tilde u_t^\top W_t q  + \log(1 + \ch(\tilde v_t^\top W_t q)).$$
Now that in the problem above $q$ is always applied to $W_t$, so in the case that $W_t$ is not full rank then all the solutions of the form $q_t = q^0_t +  \zeta$ with $\zeta \in (\operatorname{span} W_t^\top W_t)^\bot$ are admissible and leading to the same $r_t$. Then we can restrict the problem above as
$$ q_t ~~=~~  \argmin{q \in \operatorname{span} W_t^\top W_t} ~\|W_t q\|^2 - 2\tilde u_t^\top W_t q  + \log(1 + \ch(\tilde v_t^\top W_t q)).$$
To conclude, take the economic eigenvalue decomposition of $W_t^\top W_t$, i.e., $W_t^\top W_t = U_t \Sigma_t U_t^\top$ with $U_t \in \R^{2\times p_t}$ with $p_t$ the rank of $W_t^\top W_t$, such that $U_t^\top U_t = I$ and $\Sigma_t \in \R^{p_t \times p_t}$ is diagonal and positive \cite{golub2012matrix}. Now we consider the substitution $\omega = \Sigma_t^{1/2} U_t^\top q$, whose inverse is $q = U_t \Sigma_t^{-1/2} \omega$ since $q\in \operatorname{span} W_t^\top W_t$ and $U_tU_t^\top$ is the projection matrix whose span is exactly $\operatorname{span} W_t^\top W_t$, i. e. $U_t  U_t^\top q = q$ for any $q\in \operatorname{span} W_t^\top W_t$, which leads to the equivalent problem 
$$\omega_t ~~=~~  \argmin{\omega \in \R^{p_t}} ~\omega^\top\omega - 2 u_t^\top \omega  + \log(1 + \ch(v_t^\top \omega)),$$ 
where 
\begin{align*}
u_t &= \Sigma_t^{-1/2} U_t^\top W_t^\top \tilde u_t = \Sigma_t^{-1/2} U_t^\top W_t^\top W_t e_1 = \Sigma_t^{1/2} U_t^\top e_1,\\
v_t &= \Sigma_t^{-1/2} U_t^\top W_t^\top \tilde v_t = \Sigma_t^{-1/2} U_t^\top W_t^\top W_t e_2 = \Sigma_t^{1/2} U_t^\top e_2.
\end{align*}
Note that in particular $\omega_t = \Sigma_t^{1/2} U_t q_t$ and $q_t = U_t \Sigma_t^{-1/2} \omega_t$. Then 
$$\htheta_t = L_{t-1}^{-\top} r_t = L_{t-1}^{-\top} W_t q_t = L_{t-1}^{-\top} W_t U_t \Sigma_t^{-1/2}\omega_t.$$
\end{proof}

\begin{proof}{\bf of Lemma~\ref{lm:solve-small-prob}}
Since $\Omega_t$ is smooth and strongly convex, we can apply standard results on gradient descent (see for example Theorem 3.10 of \cite{bubeck2015convex}), obtaining
$$\|\omega^T_t -\omega_t\| \leq e^{-T/(2\kappa_t)}\|\omega^0_t - \omega_t\|,$$
when gradient descent is used with step-size $\gamma = 1/\beta_t$ and where $\kappa_t = \beta_t / \alpha_t$ with $\alpha_t$ a lower bound of the strong convexity constant of $\Omega_t$ and $\beta_t$ an upper bound the Lipschitz constant of $\nabla \Omega_t$. Note indeed that if $\Omega_t$ is $\alpha$-strongly convex for some $\alpha$, it will be also $\alpha'$-strongly convex, for any $0 < \alpha' \leq \alpha$; moreover if $\nabla \Omega_t$ is $\beta$-Lipschitz for some $\beta$, it will be also $\beta'$-Lipschitz, for any $\beta' \geq \beta$; for more details see Chapter 3.4 of \cite{bubeck2015convex}. Now, by construction $\alpha_t = 1$, indeed $\Omega_t(\omega) - \|\omega\|^2$ is still a convex problem. Moreover, for any $\omega, \omega' \in \R^2$, by the mean value theorem applied to the function $g: [0,1] \to \R^2$ defined as $g(r) = \nabla \Omega(\omega + r (\omega' - \omega))$, there exists a $q \in \R^2$ such that
$$ \nabla \Omega_t(\omega) - \nabla \Omega_t(\omega') = \nabla^2 \Omega_t(q) (\omega' - \omega).$$
This implies that $\|\nabla \Omega_t(\omega) - \nabla \Omega_t(\omega')\| \leq \sup_q \|\nabla^2 \Omega_t(q)\| \|\omega' - \omega\|$ so the Lipschitz constant of $\nabla \Omega_t$ is upper bounded by $\beta_t = \sup_q \|\nabla^2 \Omega_t(q)\|$.
The Hessian of $\Omega_t$ is defined as
$$\nabla^2 \Omega_t(\omega) = 2I +  \frac{1}{1+\cosh(v_t^\top \omega)} v_t v_t^\top,$$
then 
$$\sup_q \|\nabla^2 \Omega_t(q)\| \leq 2 + \|v_t\|^2 \sup_w \frac{1}{1+\cosh(v_t^\top w)} \leq  2 + \frac{\|v_t\|^2}{2}.$$
To conclude, note that $v_t$ in Thm.~\ref{thm:dec-htheta-alg} is defined as $v_t =  (W_t^\top W_t)^{1/2} e_2$ with $e_2 = (0,1)$, $W_t = L_{t-1}^{-1} (b_{t-1}, x_t)$, $L_{t-1}$ the lower triangular Cholesky decomposition of $A_{t-1}$ (i.e. $A_{t-1} = L_{t-1} L_{t-1}^\top$) and $A_{t-1}, b_{t-1}$ defined in Eq.~\eqref{eq:def-At-bt}. Then
$$\|v_t\|^2 = v_t^\top v_t = e_2 U_t \Sigma_t U_t^\top e_2^\top = e_2 W_t^\top W_t e_2 = x_t^\top L_{t-1}^{-\top} L_{t-1}^{-1} x_t  = x_t^\top A_{t-1}^{-1} x_t.$$
So $\|v_t\|^2 \leq \|x_t\|^2 \|A_{t-1}\|^{-1} \leq R^2/\lambda$, since $\|x_t\| \leq R$ by assumption and $A_{t-1} \succeq \lambda I$ by construction. Finally $\beta_t = 2 + \|v_t\|^2/2 \leq 2 + R^2/(2\lambda)$ and $\alpha_t = 1$, then $\kappa_t \leq 2 + R^2/(2\lambda)$ and $\gamma_t = 1/(2 + R^2/(2\lambda))$. We have
$$\|\omega^T_t - \omega_t\| \leq \exp(-T/(2\kappa_t) + \log \|) \leq \exp(-T/(4 + R^2/\lambda) + \log \|\omega^0_t - \omega_t\|).$$
To quantify $\|\omega^0_t - \omega_t\|$ we need a bound for $\|\omega_t\|$. Note that, since $\Omega_t$ is smooth and convex, $\omega_t$ is characterized by $\nabla \Omega_t(\omega_t) = 0$, i.e. $2\omega_t - 2u_t - (1 + e^{v_t^\top \omega_t})^{-1} v_t + (1 + e^{-v_t^\top \omega_t})^{-1} v_t = 0$, from which
$$\|\omega_t\| \leq \|u_t\| + \sup_q \left|(1 + e^{-v_t^\top q})^{-1} - (1 + e^{v_t^\top q})^{-1} \right| \|v_t\|/2 \leq \|u_t\| + \|v_t\|/2.$$
Analogously to the case of $v_t$, by definition of $u_t$, we have $\|u_t\|^2 = e_1 W_t^\top W_t e_1 = b_{t-1} A_{t-1}^{-1} b_{t-1},$ then
$\|u_t\|^2 \leq \|b_{t-1}\|^2 \|A_{t-1}^{-1}\| \leq \|b_{t-1}\|^2/\lambda$.
Now we need a bound for $b_{t-1}$. Note that for any $s \in \{1,\dots, t-1\}$, we have $g_s = \nabla \ell_s(\htheta_s) = -(1 + e^{y_s \htheta_s^\top x_s})^{-1} y_s x_s$, moreover $\eta_s = e^{y_s \htheta_s^\top x_s}/(1+BR)$ and
$$(\eta_s g_s^\top \htheta_s - 1) g_s  = \frac{y_s x_s^\top \htheta_s}{2 + 2 \cosh(y_s x_s^\top \htheta_s)} (1+BR)^{-1} y_s x_s + \frac{1}{1+e^{y_s x_s^\top \htheta_s}} y_s x_s.$$
Since $\sup_z |z/(2 + 2 \cosh(z))| \leq 1$, $y_s \in \{-1,1\}$ and $\|x_s\| \leq R$ by assumption, we have $\|(\eta_s g_s^\top \htheta_s - 1) g_s\| \leq ((1+BR)^{-1} + 1) R  \leq 2R$, then 
\begin{equation}\label{eq:bound-b-t-1}
\|b_{t-1}\| = \frac{1}{2}\sum_{s=1}^{t-1}\|(\eta_s g_s^\top \htheta_s - 1) g_s \| \leq (t-1) R.
\end{equation}
To conclude, $\|u_t\| \leq R\lambda^{-1/2}(t-1)$, $\|\omega_t\| \leq \|u_t\| + \|v_t\|/2 \leq R\lambda^{-1/2}t$. By choosing $\omega^0_t = 0$, then $P_t = \|\omega_t\| \leq R\lambda^{-1/2}t$ and so $\|\omega^T_t - \omega_t\| \leq \epsilon$, when choosing $T \geq (4+R^2/\lambda)\log(R\lambda^{-1/2}t/\epsilon)$.
\end{proof}

\begin{proof}{\bf of Theorem~\ref{thm:comp-compl}}
We first analyze the cost of one iteration of Algorithm~\ref{algorithm-short} (which is detailed in Algorithm~\ref{algorithm-detail} presented above). Note that at each step $t$, the cost of the gradient descent algorithm performed to compute $\omega^T_t$ is the number of iterations $T$, since we are solving a $p_t$-dimensional problem, with $p_t \in  \{1,2\}$. The two most expensive operation performed at step $t$ (excluding gradient descent) are the solution of triangular linear systems of dimensions $d \times d$ when computing $L_{t-1}^{-1} v$ or $L_{t-1}^{-\top} v$ for some vector $v \in \R^d$, which costs $O(d^2)$ (this operation is performed 4 times). The other expensive operation is the rank 1 Cholesky update of $L_{t-1}$ with the vector $\sqrt{\eta_t/2}g_t$, which costs $O(d^2)$ \cite{golub2012matrix}, indeed the eigendecomposition is performed on the matrix $W_t^\top W_t$ which is $2 \times 2$.
By repeating such operation for $n$ steps, we obtain a total cost of 
$$O(n d^2 + n T).$$

The upper-bound on the regret is a direct consequence of Theorem \ref{thm:main_theorem_approx} and Lemma \ref{lm:solve-small-prob}, with $T$ chosen according to the lemma and $\epsilon = \frac{\sqrt{\lambda}}{3nR \left(\frac{nR^2}{8\lambda} + B \right)}$, since $\|L_{t-1}^{-1}\|^2 = \|A_{t-1}^{-1}\| \leq \lambda^{-1}$ and $W_t U_t \Sigma_t^{-1/2}$ is a partial isometry, we have
\begin{equation} \|\hat \theta_t - \tilde \theta_t\| \leq \|L_{t-1}^{-1}\| \|W_t U_t \Sigma_t^{-1/2}\| \|\omega_t - \omega_t^T\| \leq \lambda^{-1/2} \epsilon \leq \frac{1}{3nR \left(\frac{nR^2}{8\lambda} + B \right)} ,
\end{equation}
that plugged in the result of Theorem \ref{thm:main_theorem_approx} gives the desired result.

\end{proof}

\end{document}